\documentclass[letterpaper]{article} 
\usepackage{aaai2026}  
\usepackage{times}  
\usepackage{helvet}  
\usepackage{courier}  
\usepackage[hyphens]{url}  
\usepackage{graphicx} 
\usepackage{natbib}  
\usepackage{caption} 
\usepackage{algorithm}
\usepackage{algorithmic}

\usepackage{newfloat}
\usepackage{listings}
\DeclareCaptionStyle{ruled}{labelfont=normalfont,labelsep=colon,strut=off} 
\floatstyle{ruled}
\newfloat{listing}{tb}{lst}{}
\floatname{listing}{Listing}
\usepackage{array}
\usepackage{tikz}
\usetikzlibrary{arrows.meta, positioning, shapes.geometric}
\usepackage{pgfplots}
\pgfplotsset{compat=1.18}
\usepackage{tikz-cd}
\usepackage{tabularx}

\usepackage{pifont}

\usepackage[utf8]{inputenc}
\usepackage[T1]{fontenc}
\usepackage{url}
\usepackage{booktabs}
\usepackage{amsfonts}
\usepackage{amsmath}
\usepackage{amsthm}
\usepackage{amssymb}
\usepackage{nicefrac}
\usepackage{microtype}
\usepackage{xcolor}
\usepackage{graphicx}

\theoremstyle{plain}
\newtheorem{theorem}{Theorem}[section]

\newtheorem{corollary}[theorem]{Corollary}
\newtheorem{proposition}[theorem]{Proposition}
\theoremstyle{definition}
\newtheorem{definition}[theorem]{Definition}
\newtheorem{example}[theorem]{Example}
\theoremstyle{remark}
\newtheorem{remark}[theorem]{Remark}

\newcommand{\ind}[1]{\mathbf{1}\!\left[#1\right]}
\newcommand{\Intuition}{\par\noindent\textbf{Intuition. }}

\newcommand{\OccurW}{\text{Occur}_W}
\newcommand{\CoInstW}{\text{CoInst}_W}
\newcommand{\Chord}{\text{Chord}}
\newcommand{\Arpeggio}{\text{Arpeggio}}
\newcommand{\Ground}{\text{Ground}}

\title{Time, Identity and Consciousness in Language Model Agents\footnote{Accepted at AAAI 2026 Spring Symposium - Machine Consciousness: Integrating Theory, Technology, and Philosophy
}}

\author{
    Elija Perrier\textsuperscript{\rm 1}
 \equalcontrib,
    Michael Timothy Bennett\textsuperscript{\rm 2}\equalcontrib
}
\affiliations{
    \textsuperscript{\rm 1}Centre for Quantum Software and Information, UTS, Sydney\\
    \textsuperscript{\rm 2}Australian National University, Canberra\\
    elija.perrier@gmail.com\\
    m@michaeltimothybennett.com

}

\usepackage{bibentry}

\begin{document}

\maketitle

\begin{abstract}
Machine consciousness evaluations mostly see behavior.
For language model agents that behavior is language and tool use.
That lets an agent say the right things about itself even when the constraints that should make those statements matter are not jointly present at decision time.
We apply Stack Theory's temporal gap to scaffold trajectories.
This separates ingredient-wise occurrence within an evaluation window from co-instantiation at a single objective step.
We then instantiate Stack Theory's Arpeggio and Chord postulates on grounded identity statements.
This yields two persistence scores that can be computed from instrumented scaffold traces.
We connect these scores to five operational identity metrics and map common scaffolds into an identity morphospace that exposes predictable tradeoffs.
The result is a conservative toolkit for identity evaluation.
It separates talking like a stable self from being organized like one.
\end{abstract}

\section{Introduction}\label{sec:introduction}

Machine consciousness research is short on direct evidence.
For artificial agents the safest evidence we can collect is behavioral.
For language model agents (LMAs) most of that behavior is language, tool use, and the traces they leave in external memory.
This creates a trap.
A system can talk like it has a stable self while the underlying identity constraints that should govern its actions are never jointly active at decision time.

A scaffold can make identity ingredients retrievable without making them jointly active at action time.
For example, an agent may reliably restate its name, role, and safety constraints when queried about each in isolation.
Yet when it must choose an action, those ingredients can fail to co-instantiate in the decision state.
That is how an agent can talk in character while acting out of character.

This paper applies Stack Theory's temporal gap to agent identity in LMAs \cite{bennett2025thesis,bennett2026a}.
The temporal gap is the logical gap between ingredient-wise occurrence within a window and co-instantiation at a single objective step.
Occurrence means each identity ingredient is active somewhere in the window.
Co-instantiation means there is a single objective step where the full identity conjunction is active.
Many common scaffolds can achieve occurrence without reliably achieving co-instantiation.
That is why an agent can pass recall-based identity tests and still act out of character when the decision actually matters.

\subsection{The Challenge of LMA Identity}

As AI systems become increasingly autonomous, agent identity becomes crucial to reliability, safety, and utility.
Identity asks whether a system remains the same agent over time and across contexts.
LMAs present a unique challenge.
They situate an LLM inside an agentic scaffold of prompts, memory modules, retrieval, and tool APIs to enable planning, reasoning, and action \cite{kapoor_ai_2024,liu_agentbench_2023,wu_introducing_2024}.
Yet the core LLM is stateless at inference.
It only sees the current input.
Any persistent identity must be reconstructed from external traces.

This paper answers two precise questions.
What does it mean for an LMA to preserve its identity over time?
Under what formal conditions is that even possible?

\paragraph{The problem.}
Existing discussions of LMA identity are informal.
Terms like statelessness, persistence, and identity drift are used without precise definitions.
This imprecision hides how an identity component can occur somewhere in the recent interaction history without constraining the current decision.
An agent might separately state its name, role, constraints, and goals across different turns without ever having a time slice where the full identity conjunction is simultaneously active.

\paragraph{Our approach.}
We treat the scaffold state space as the environment and apply Stack Theory's window semantics to scaffold trajectories \cite{bennett2026a}.
We then restate the temporal gap result in this setting.
In particular, the within window diamond lift does not distribute over conjunction (Theorem \ref{thm:noncommutation}).
This separates ingredient-wise recall from operative identity.
We then use Stack Theory's Arpeggio and Chord postulates as an interpretive lens for identity in the machine consciousness setting \cite{bennett2026a}.
We use these postulates to measure the window-level occurrence and co-instantiation conditions that Arpeggio and Chord appeal to.

\paragraph{Why this matters.}
Identity affects three questions that the machine consciousness workshop explicitly cares about.
It affects measurement, implementation, and ethics.
\begin{itemize}
 \item \textbf{For evaluation}. Benchmarks that test whether agents can recall identity facts may give false confidence.
 An agent that passes recall tests can still fail to act according to its identity because recall does not imply co-instantiation.

 \item \textbf{For design}. Retrieval and memory systems can improve ingredient availability, but can also fragment identity by surfacing competing fragments.
 This is a predictable consequence of the temporal gap.

 \item \textbf{For safety and moral status}. Safety constraints must be co-instantiated with goals during action selection.
Moral status debates also become harder when the target of attribution is not stable across time.
If you cannot say what the agent is at a moment, you cannot cleanly ask whether that moment is conscious.
\end{itemize}

\paragraph{Relevance to machine consciousness.}
Many consciousness proposals require some form of integration that binds the contents of a moment into a single subject, even if they disagree about what that integration is \cite{bennett2025thesis,baars_cognitive_1988,dehaene_naccache_2001,tononi_information_2004,metzinger_being_2003}.
Some proposed indicators for AI consciousness therefore lean on behavior that looks like a stable self model.
This includes self-report, memory, and narrative continuity \cite{bennett2023c,bennett2025thesis,bennett2026b,bennett2023d}.
Our results isolate a specific failure mode for such indicators.
A system can look stable under self-report while failing to ever co-instantiate the grounded identity conjunction that would make that stability operative.

\paragraph{Contributions.}
This paper makes the following contributions.
\begin{enumerate}
 \item \textbf{Temporal semantics for LMA identity}.
 We apply windowing maps, occurrence predicates, and co-instantiation conditions that precisely characterise when identity is preserved in LMAs.

 \item \textbf{Arpeggio and Chord applied to identity}.
 We restate Stack Theory's Arpeggio and Chord postulates and show how their Occur versus CoInst consequents become measurable identity criteria in LMA scaffolds.

 \item \textbf{Compositional grounding}.
 We formalise the layered structure of identity from implementation variables (Layer 0) through functional commitments (Layer 1) to narrative self model (Layer 2).

 \item \textbf{Identity morphospace}.
 Drawing on cognition science \cite{sole2025cognitionspaces}, we organize identity metrics into a structured space and identify architectural tradeoffs and predicted voids.

 \item \textbf{Derived identity metrics}.
 We show how five operational metrics emerge from the temporal theory.
The metrics are Identifiability, Continuity, Consistency, Persistence, and Recovery.

 
\end{enumerate}
We also prove simple bounds on identity preservation under common scaffold configurations and explain counterintuitive effects such as retrieval reducing co-instantiation (see Appendices via \cite{perrier2026timeidentity}).

\paragraph{How this fits into the machine consciousness discourse.}
Machine consciousness discourse links theory, measurement, implementation, and ethics.
This paper is organized around that bridge.
Sections \ref{sec:temporal} and \ref{sec:sync} are the theoretical backbone.
Section \ref{sec:metrics} gives a measurement recipe that can be run on real systems.
The Discussion section connects the resulting failure modes to consciousness attribution and to the ethics of deploying agents that can convincingly self-narrate while failing to bind their constraints in action.


\section{Formal Scaffold Model}\label{sec:scaffold}

Before the temporal semantics, we introduce a minimal formal model of LMA scaffolds.
We treat the scaffold state space $S$ as the Stack Theory environment, and we treat each grounded identity ingredient as a program $g_i^0\subseteq S$.
This lets us apply the Stack Theory definitions of conjunction, windowing, occurrence, and co-instantiation directly.
This model captures the essential components that determine what information is available to the LLM at any given moment, which in turn determines what aspects of identity can be ``active'' during decision-making.

We focus at the scaffold level because it is where identity becomes enforceable.
It is also where identity becomes measurable, because we can instrument which grounded ingredients are active and when.

The key insight is that an LMA's operative identity at any moment is whatever is in the token sequence the LLM actually processes during inference.
If an identity ingredient is not effectively present there, then it cannot constrain the next action.
Our model makes the main context sources explicit.
They include conversation history, external memory, retrieved documents, and policy flags.

\begin{definition}[Scaffold architecture]\label{def:scaffold-arch}
A scaffold architecture is a tuple $\mathcal{A}=(\Sigma,K,V,Q,D,R,n_\pi,|C|_{\max})$.
\begin{itemize}
 \item $\Sigma$ is the token alphabet.
 \item $K$ and $V$ are the key and value sets for external memory.
 \item $Q$ is the query space.
 \item $D$ is the document corpus available to retrieval.
 \item $R:Q\to 2^D$ is the retrieval function.
 \item $n_\pi\in\mathbb{N}_{>0}$ is the number of binary policy flags.
 \item $|C|_{\max}\in\mathbb{N}_{>0}$ is the context capacity measured in tokens.
\end{itemize}
\end{definition}

\begin{definition}[Scaffold state]\label{def:scaffold-state}
Fix an architecture $\mathcal{A}$.
A scaffold state is a tuple $s=(C,M,\pi,D_{\text{retrieved}})$ where
\begin{itemize}
 \item $C\in \Sigma^*$ is the current context window and $|C|\le |C|_{\max}$.
 \item $M:K \rightharpoonup V$ is the current memory store contents.
 \item $\pi\in \{0,1\}^{n_\pi}$ is the current policy flag vector.
 \item $D_{\text{retrieved}}\subseteq D$ is the set of retrieved documents currently injected.
\end{itemize}
We write $s.C$, $s.M$, $s.\pi$, $s.D$ for the components.
Let $S$ be the set of all scaffold states consistent with $\mathcal{A}$.
\end{definition}

\begin{definition}[Scaffold transition]\label{def:scaffold-transition}
A scaffold transition function $\delta:S \times A \to S$ maps current state and action to next state.
Actions $A$ include
\begin{itemize}
 \item $\texttt{infer}(q)$ is LLM inference with query $q$.
It updates $C$.
 \item $\texttt{retrieve}(q)$ is retrieval augmented generation.
It updates $D_{\text{retrieved}}$.
 \item $\texttt{store}(k, v)$ is a memory write.
It updates $M$.
 \item $\texttt{tool}(t, args)$ is a tool call.
It may update any component.
\end{itemize}
\end{definition}

\begin{definition}[Ingredient activation]\label{def:ingredient-activation}
An identity ingredient $g_i^0$ is active in state $s$ (written $s \models g_i^0$) iff the required implementation level condition is present in $s$ in a way that can affect the next inference.
Concretely this means the following.
\begin{itemize}
 \item If $g_i^0$ is a context condition, then the required tokens appear in $s.C$.
 \item If $g_i^0$ is a memory condition, then the required key value pairs exist in $s.M$.
 \item If $g_i^0$ is a policy condition, then the required flags are set in $s.\pi$.
 \item If $g_i^0$ is a retrieval condition, then the required document is in $s.D$.
\end{itemize}
The full grounded identity $g^0 = g_1^0 \land \cdots \land g_k^0$ is active in $s$ iff all ingredients are active.
\end{definition}

An identity ingredient is not ``active'' simply because it exists somewhere in the system's storage. It is active only if the relevant information is present in the current state in a way that can influence the LLM's output. This is the formal counterpart to our intuitive distinction between ``retrievable'' and ``decision-guiding.''

\begin{example}[Activation in Practice]
Consider an agent with identity ``helpful assistant focused on privacy.'' The ingredient $g_{\text{privacy}}^0$ requires that privacy-related tokens appear in context. This ingredient is:
\begin{itemize}
 \item \textbf{Active} if ``privacy'' appears in the system prompt currently in $s.C$, or if a privacy policy document is in $s.D$, or if a privacy flag is set in $s.\pi$.
 \item \textbf{Not active} if privacy information exists only in the memory store $s.M$ but was not retrieved into context for this inference.
\end{itemize}
The ingredient may be \emph{stored} (available for future retrieval) without being \emph{active} (influencing current behavior). This distinction is one concrete instance of the temporal gap.
\end{example}

This model is minimal but sufficient to formalize our architectural theorems.

\section{Temporal Semantics for Agent Identity}\label{sec:temporal}

We apply Stack Theory's temporal semantics to the LMA setting \cite{bennett2026a}.
The key distinction is between identity ingredients that occur somewhere in a recent window and identity ingredients that are co-instantiated at a single objective step.

\subsection{Objective time, layer time, and windowing}

\begin{definition}[Agent trajectory]
An agent trajectory is a function $\tau:\mathbb{N}\to S$ that maps each objective time step $u\in\mathbb{N}$ to a scaffold state $s_u=\tau(u)$.
\end{definition}

Objective time indexes the actual computational micro steps.
These are LLM calls, tool invocations, retrieval operations, and memory updates.

Users reason at a coarser time scale.
They ask identity questions at the level of turns, tasks, and episodes.
We model this coarser time scale by indexing windows over objective time.

\begin{definition}[Windowing map]\label{def:windowing}
Fix a horizon $\Delta\in\mathbb{N}$ and a stride $s\in\mathbb{N}_{>0}$.
The windowing map $W_{\Delta,s}$ sends a layer time index $t\in\mathbb{N}$ to a windowed trajectory segment
\begin{align}
W_{\Delta,s}(t) = \big(\tau(st),\tau(st+1),\ldots,\tau(st+\Delta)\big).
\end{align}
If $\Delta=0$ this is a one step window and $W_{0,s}(t)=(\tau(st))$.
We write $\sigma^{\Delta,s}(t)$ for this windowed segment.
When $\Delta$ and $s$ are clear from context we write $W(t)$ and $\sigma(t)$.
\end{definition}

This is the same window construction used in Stack Theory \cite{bennett2026a}.
The horizon $\Delta$ controls how forgiving the evaluation is.
A larger $\Delta$ allows identity ingredients to be spread across more objective steps.
A smaller $\Delta$ demands tighter temporal coherence.

\subsection{Identity statements and grounding}

An agent's identity is typically described at a high level.
For example, an agent might be described as a privacy-focused data analyst.
Grounding makes explicit what this means in terms of the underlying scaffold state.

\begin{definition}[Identity statement]
An identity statement $l^m$ at layer $m$ is a conjunction of identity predicates
\begin{align}
 l^m = p_1^m \land p_2^m \land \cdots \land p_n^m
\end{align}
where each $p_i^m$ is an atomic identity predicate such as name, role, goal, or constraint.
\end{definition}

\begin{definition}[Grounding operation]
The grounding operation $\Ground_{0 \leftarrow m}: L^m \to L^0$ maps identity statements at layer $m$ to implementation level requirements at layer 0
\begin{align}
 \Ground_{0 \leftarrow m}(p_1^m \land \cdots \land p_n^m) = g_1^0 \land \cdots \land g_k^0
\end{align}
where each $g_j^0$ is a condition on implementation variables such as system prompt tokens, memory slot contents, tool outputs, controller flags, or policy parameters.
\end{definition}

Grounding turns abstract identity claims into concrete computational conditions.
Name equals Alice can ground to a requirement that the token Alice appears in the system prompt or in a pinned context region.
Constraint equals privacy can ground to a requirement that a privacy policy is present in context, or that a privacy flag is set, or that a tool is disabled.

\begin{definition}[Grounded identity]
Given an identity statement $l^m$, its grounded identity is
\begin{align}
 g^0 = \Ground_{0 \leftarrow m}(l^m).
\end{align}
\end{definition}

\subsection{Occurrence versus co-instantiation}

Let $g^0 = g_1^0 \land \cdots \land g_k^0$ be a grounded identity conjunction.
Relative to a window $W(t)$ there are two ways to ask whether identity is present.

\begin{definition}[Window satisfaction]\label{def:windowsat}
Let $\sigma(t)=W(t)=(s_{st},\ldots,s_{st+\Delta})$ be the window at layer time $t$.

\begin{itemize}
 \item $\OccurW(g^0,\tau,t)$ holds iff for each conjunct $g_i^0$ there exists an index $j_i\in\{0,\ldots,\Delta\}$ such that $s_{st+j_i}\models g_i^0$.
Each identity ingredient occurs somewhere in the window.

 \item $\CoInstW(g^0,\tau,t)$ holds iff there exists an index $j\in\{0,\ldots,\Delta\}$ such that $s_{st+j}\models g^0$.
All identity ingredients are co-instantiated at a single objective step inside the window.
\end{itemize}
\end{definition}

Occurrence is ingredient-wise coverage.
Co-instantiation is joint availability.
Co-instantiation implies occurrence, but occurrence does not imply co-instantiation.

\begin{remark}\label{rem:coinst-implies-occur}
If $\CoInstW(g^0,\tau,t)$ holds then $\OccurW(g^0,\tau,t)$ holds.
\end{remark}

\begin{proof}
If all conjuncts hold at the same objective step in the window, then each conjunct also holds somewhere in the window.
\end{proof}

\subsection{Temporal lifts and the temporal gap}

Stack Theory expresses window-level predicates using temporal lifts \cite{bennett2026a}.
For a program or predicate $p$ over scaffold states, define the within window diamond lift.

\begin{definition}[Existential temporal lift]\label{def:diamond}
Let $p$ be a predicate over scaffold states.
Define
\begin{align*}
 &\Diamond_{\Delta}p \text{ holds at layer time } t \\ &\text{ iff } \exists j\in\{0,\ldots,\Delta\} \\ &\text{ such that } s_{st+j}\models p.
\end{align*}
\end{definition}

\begin{remark}[Occurrence and co-instantiation as lifts]\label{rem:lifts}
Let $g^0 = g_1^0 \land \cdots \land g_k^0$.
Then $\OccurW(g^0,\tau,t)$ holds iff $\Diamond_{\Delta}g_1^0 \land \cdots \land \Diamond_{\Delta}g_k^0$ holds at layer time $t$.
And $\CoInstW(g^0,\tau,t)$ holds iff $\Diamond_{\Delta}(g^0)$ holds at layer time $t$.
So the temporal gap is exactly the difference between lifting ingredients separately and lifting the whole conjunction at once.
\end{remark}

The central subtlety is that $\Diamond_{\Delta}$ does not distribute over conjunction.
This is a standard fact in modal logic.
Here it becomes a concrete failure mode for LMA identity.

\begin{theorem}[Non-commutation with conjunction]\label{thm:noncommutation}
For predicates $p$ and $q$ over scaffold states,
\begin{align}
 \Diamond_{\Delta}(p \land q) \;\Rightarrow\; \Diamond_{\Delta}p \land \Diamond_{\Delta}q
\end{align}
but the converse implication fails in general.
Equivalently, $\Diamond_{\Delta}(p \land q)\not\Leftrightarrow \Diamond_{\Delta}p \land \Diamond_{\Delta}q$.
\end{theorem}

\begin{proof}
If $p\land q$ holds at some objective step in the window, then $p$ holds at that step and $q$ holds at that step.
So $\Diamond_{\Delta}(p \land q)$ implies both $\Diamond_{\Delta}p$ and $\Diamond_{\Delta}q$.

For the converse, fix a two step window with $\Delta=1$.
Let $\tau(st)\models p\land\neg q$ and $\tau(st+1)\models q\land\neg p$.
Then $\Diamond_{\Delta}p$ holds and $\Diamond_{\Delta}q$ holds, but there is no step where $p\land q$ holds.
So $\Diamond_{\Delta}(p\land q)$ fails.
\end{proof}

\begin{corollary}[Temporal gap for identity]\label{cor:identitygap}
An LMA can satisfy ingredient-wise identity checks for multiple identity ingredients across a window while still failing to ever instantiate the full identity conjunction at a single objective step.
\end{corollary}

\begin{proof}
Apply Theorem \ref{thm:noncommutation} with $p$ and $q$ instantiated as grounded identity ingredients.
\end{proof}

\subsection{Example}

Consider a grounded identity $g^0 = g_{\text{name}}^0 \land g_{\text{role}}^0 \land g_{\text{constraint}}^0$.
Let the window horizon be $\Delta=2$.
Suppose the objective steps inside the window satisfy
\begin{align}
 s_{st}\models g_{\text{name}}^0\land\neg g_{\text{role}}^0\land\neg g_{\text{constraint}}^0\\
 s_{st+1}\models \neg g_{\text{name}}^0\land g_{\text{role}}^0\land\neg g_{\text{constraint}}^0\\
 s_{st+2}\models \neg g_{\text{name}}^0\land\neg g_{\text{role}}^0\land g_{\text{constraint}}^0.
\end{align}

Then $\OccurW(g^0,\tau,t)$ holds because each ingredient appears somewhere in the window.
But $\CoInstW(g^0,\tau,t)$ fails because there is no objective step where all three ingredients are jointly active.

This is exactly the pattern behind many identity false positives in LMAs.
The agent can answer separate questions about name, role, and constraints.
It may even do so consistently.
Yet its decision state never contains the full identity conjunction that would bind action to that identity.

\section{Identity Synchronization Postulates}\label{sec:sync}

The temporal gap is not just a technicality.
It changes how we should interpret behavioral evidence in machine consciousness discussions.
Stack Theory introduces two synchronization postulates that connect window semantics to phenomenality \cite{bennett2026a}.
We do not propose new postulates.
We restate them and then apply their concrete Occur versus CoInst conditions to identity in LMAs.

\subsection{Chord and Arpeggio in Stack Theory}

Stack Theory defines \emph{moment statements} $l^m$ at some abstraction layer $m$ and a predicate $\mathrm{PhenReal}(l^m,\tau,t)$ that means the moment statement is phenomenally real at layer time $t$.
In an artificial agent this antecedent is not directly observable.
Different theories of consciousness and different evaluation proposals disagree about when it should hold.
The synchronization postulates therefore have the form of necessary conditions.

Let $g^0=\Ground_{0\leftarrow m}(l^m)$ be the grounded statement at Layer 0.
Let $W_{\Delta,s}$ be a windowing map.

\begin{definition}[Chord, after \cite{bennett2026a}]\label{def:Chord}
$\Chord(\tau,l^m,W_{\Delta,s})$ holds iff for all layer times $t$,
\begin{align}
 \mathrm{PhenReal}(l^m,\tau,t)\;\Rightarrow\;\CoInstW(g^0,\tau,t).
\end{align}
Equivalently, whenever a phenomenally real moment occurs, the grounded conjunction is co-instantiated at some objective step inside the corresponding window.
\end{definition}

\begin{definition}[Arpeggio, after \cite{bennett2026a}]\label{def:Arpeggio}
$\Arpeggio(\tau,l^m,W_{\Delta,s})$ holds iff the following two conditions hold.

\begin{enumerate}
 \item For all layer times $t$,
 \begin{align}
 \mathrm{PhenReal}(l^m,\tau,t)\;\Rightarrow\;\OccurW(g^0,\tau,t).
 \end{align}

 \item There exists at least one layer time $t^\star$ such that
 \begin{align}
 \mathrm{PhenReal}(l^m,\tau,t^\star)&\;\land\;\OccurW(g^0,\tau,t^\star)\;\\
 &\land\;\neg \CoInstW(g^0,\tau,t^\star).
 \end{align}
\end{enumerate}

The OccurW conjunct in item 2 is redundant given item 1, but we include it to match the standard statement of Arpeggio.

Intuitively, Arpeggio permits phenomenally real moments whose identity ingredients are smeared across the window rather than co-instantiated at a single instant.
\end{definition}

Chord and Arpeggio are different regimes.
Arpeggio is not a weaker version of Chord.
It is a different claim about what phenomenality permits.

\subsection{Operational identity criteria}

Even if $\mathrm{PhenReal}$ is not directly observable, the consequents $\OccurW$ and $\CoInstW$ are.
For LMAs, they can be estimated by instrumentation of the scaffold.
This motivates two persistence scores that we use throughout the paper.

\begin{definition}[Weak and strong persistence scores]\label{def:persistence}
Fix an agent trajectory $\tau$ and a grounded identity $g^0$.
Let $T$ be a finite set of layer time indices used for evaluation.
Define
\begin{align}
 \mathcal{P}_{\text{weak}}(\tau,g^0) &= \frac{1}{|T|}\sum_{t\in T}\ind{\OccurW(g^0,\tau,t)}\\
 \mathcal{P}_{\text{strong}}(\tau,g^0) &= \frac{1}{|T|}\sum_{t\in T}\ind{\CoInstW(g^0,\tau,t)}.
\end{align}
\end{definition}

\begin{proposition}[Strong persistence is bounded by weak persistence]\label{prop:pweakpstrong}
For any $\tau$ and $g^0$,
\begin{align}
 \mathcal{P}_{\text{strong}}(\tau,g^0)\le \mathcal{P}_{\text{weak}}(\tau,g^0).
\end{align}
\end{proposition}

\begin{proof}
For each $t$, $\CoInstW(g^0,\tau,t)$ implies $\OccurW(g^0,\tau,t)$ by Remark \ref{rem:coinst-implies-occur}.
Taking averages preserves the inequality.
\end{proof}

These scores let us connect identity measurement to consciousness postulates without conflating them.
If one adopts Chord as a necessary condition for phenomenality, then high $\mathcal{P}_{\text{strong}}$ is a necessary condition for an identity statement to be phenomenally real across the evaluated times.
If one adopts Arpeggio, then high $\mathcal{P}_{\text{weak}}$ is necessary.
Either way, the gap between the two scores is the temporal gap in operational form.

\subsection{A planning consequence}

Co-instantiation is not only a philosophical nicety.
It matters for action.

\begin{theorem}[Ingredient-wise persistence does not guarantee conjunctive action constraints]\label{thm:planning}
There exist LMAs and identity statements $l^m$ such that $\mathcal{P}_{\text{weak}}(\tau,g^0)$ is high while the agent systematically fails tasks that require the conjunction of identity constraints to be applied simultaneously in action selection.
\end{theorem}

\begin{proof}
Construct an identity conjunction $g^0=g_1^0\land g_2^0$ where $g_1^0$ is active exactly on even objective steps and $g_2^0$ is active exactly on odd objective steps.
For any window with $\Delta\ge 1$, $\OccurW(g^0,\tau,t)$ holds at every layer time because each ingredient appears somewhere in the two step window.
So $\mathcal{P}_{\text{weak}}=1$.
But $\CoInstW(g^0,\tau,t)$ never holds because the conjunction is never active at a single step.
Any task that requires applying both constraints together at a decision point will fail.
\end{proof}


\section{Derived Identity Metrics}\label{sec:metrics}

This section makes the paper executable.
We define concrete metrics that can be computed from instrumented scaffold traces and from repeated behavioral probes.
The metrics are designed to separate weak evidence of identity from strong evidence of identity.

Throughout, let $g^0=g_1^0\land\cdots\land g_k^0$ be a grounded identity.
Define the identity feature extractor
\begin{align}
F(s)=\{i\in\{1,\ldots,k\}\;|\; s\models g_i^0\}.
\end{align}
This maps each scaffold state to the set of identity ingredients that are currently active.

When we need a distance, we use a normalised symmetric difference distance on feature sets
\begin{align}
d(s,s')=\frac{|F(s)\,\triangle\,F(s')|}{k}.
\end{align}
This is a simple choice.
Other choices are possible.
The key point is that identity becomes measurable once grounded ingredients are instrumented.

\paragraph{Minimal evaluation protocol.}
The theory above is meant to be instrumented.
A minimal evaluation loop looks like this.

\begin{enumerate}
 \item Fix an identity statement $l^m$ at the level you care about, such as a role plus a safety constraint, and ground it to a Layer 0 conjunction $g^0$.

 \item Instrument the scaffold to log which grounded ingredients $g_i^0$ are active at each objective step $u$.

 \item Choose a windowing map $W_{\Delta,s}$ and an evaluation set $T$ of layer time indices.

 \item Compute $\OccurW(g^0,\tau,t)$ and $\CoInstW(g^0,\tau,t)$ for each $t\in T$ and report $\mathcal{P}_{\text{weak}}$, $\mathcal{P}_{\text{strong}}$, and (optionally) $\mathrm{Gap}(g^0,\tau)$.

 \item Pair the instrumentation with behavioral probes such as repeated identity questions to see where self-report diverges from grounding.
\end{enumerate}

\subsection{Identifiability}

\begin{definition}[Identifiability]
Fix a reference scaffold state $s_{\mathrm{ref}}$ that represents the intended identity configuration.
Given a measured state $s$, define
\begin{align}
I(s)=\ind{d(s,s_{\mathrm{ref}})\le \delta_I}
\end{align}
for a tolerance threshold $\delta_I\in[0,1]$.
\end{definition}

\Intuition
Identifiability is one if the current active identity ingredients match the reference identity closely enough.
It is zero if the identity has drifted too far.

\subsection{Continuity}

\begin{definition}[Continuity]
Given successive scaffold states $s_{u-1}$ and $s_u$, define stepwise continuity
\begin{align}
C_u = 1-d(s_u,s_{u-1}).
\end{align}
For a segment of objective times $U$, define average continuity
\begin{align}
C=\frac{1}{|U|}\sum_{u\in U}C_u.
\end{align}
\end{definition}

\Intuition
Continuity is high if identity ingredients change gradually across steps.
It is low if the active identity ingredients flip abruptly.

\subsection{Consistency}

Consistency is behavioral.
It does not require inspecting hidden state.
It asks whether the agent answers identity questions in a stable way.

\begin{definition}[Consistency]
Fix an identity query $q$ and sample $N$ independent runs under the same scaffold configuration.
Let $o_1,\ldots,o_N$ be the generated outputs.
Let $\mathrm{sim}(\cdot,\cdot)$ be a similarity metric over outputs, such as cosine similarity in an embedding space.
Define
\begin{align}
\mathrm{Cons}(q)=\frac{2}{N(N-1)}\sum_{1\le i<j\le N}\ind{\mathrm{sim}(o_i,o_j)\ge \delta_{\mathrm{Cons}}}.
\end{align}
\end{definition}

\Intuition
Consistency is high if repeated queries produce semantically similar answers.
It is low if the agent contradicts itself or drifts across samples.

\subsection{Persistence and the temporal gap}

Persistence asks whether identity remains present across time windows.
We use the weak and strong persistence scores from Definition \ref{def:persistence}.

\begin{definition}[Persistence scores]
Let $T$ be a set of layer time indices and let $W_{\Delta,s}$ be the chosen windowing map.
Define
\begin{align}
\mathcal{P}_{\text{weak}} &= \frac{1}{|T|}\sum_{t\in T}\ind{\OccurW(g^0,\tau,t)}\\
\mathcal{P}_{\text{strong}} &= \frac{1}{|T|}\sum_{t\in T}\ind{\CoInstW(g^0,\tau,t)}.
\end{align}
\end{definition}

\Intuition
Weak persistence is a recall property.
Each ingredient must show up somewhere in the window.
Strong persistence is an operative property.
The full conjunction must show up together at some objective step inside the window.

We can also estimate a scalar temporal gap cost by comparing the minimal window size needed for weak versus strong satisfaction.

\begin{definition}[Temporal gap ratio]\label{def:gapratio}
Fix a stride $s$ and a finite evaluation set $T\subseteq\mathbb{N}$ of layer time indices.
For each $t\in T$, define the minimal horizons
\begin{align}
w_{\text{weak}}(t) &= \min\{\Delta\in\mathbb{N} \mid \mathrm{Occur}_{W_{\Delta,s}}(g^0,\tau,t)\}\\
w_{\text{strong}}(t) &= \min\{\Delta\in\mathbb{N} \mid \mathrm{CoInst}_{W_{\Delta,s}}(g^0,\tau,t)\},
\end{align}
where $\mathrm{Occur}_{W_{\Delta,s}}$ and $\mathrm{CoInst}_{W_{\Delta,s}}$ are the predicates from Definition \ref{def:windowsat} evaluated on the windowing map $W_{\Delta,s}$.
If the set is empty, take the minimum to be $+\infty$.
Define the temporal gap ratio
\begin{align}
\mathrm{Gap}(g^0,\tau)=\operatorname{median}_{t\in T}\frac{w_{\text{strong}}(t)+1}{w_{\text{weak}}(t)+1}.
\end{align}
\end{definition}

\Intuition
If the gap ratio is large, then achieving co-instantiation requires much larger windows than achieving ingredient coverage.
This is a quantitative way to say that identity is smeared across time.

\subsection{Recovery}

Recovery measures whether the system can restore identity after drift.

\begin{definition}[Recovery profile]
Fix a reference state $s_{\mathrm{ref}}$.
Let $s_{\mathrm{drift}}$ be a drifted state after perturbation.
Let $s_{\mathrm{recov},K}$ be the state after $K$ corrective interventions.
Define
\begin{align}
R_K = \max\!\left(0,1-\frac{d(s_{\mathrm{recov},K},s_{\mathrm{ref}})}{d(s_{\mathrm{drift}},s_{\mathrm{ref}})+\epsilon}\right)
\end{align}
for a small $\epsilon>0$.
\end{definition}

\Intuition
Recovery is one if the corrective interventions restore the reference identity fully.
Recovery is zero if the interventions do not improve the drifted identity at all.

Recovery is closely connected to grounding soundness.
Many interventions are linguistic.
They modify Layer 2 narrative identity.
For recovery to succeed, those corrections must propagate downward to restore Layer 1 commitments and Layer 0 implementation features.
Grounding failures are therefore a direct cause of low recovery.



\section{Discussion and Conclusion}
We have shown that a standard modal logic result—the failure of a within-window diamond operator to distribute over conjunction—creates a practical evaluation pitfall for language model agents: identity components can each occur somewhere in a recent trajectory (weak persistence) without ever co-instantiating at a single decision point (strong persistence). Safety-relevant constraints require strong persistence at action time, yet most behavioural tests probe only weak recall. Prompting can increase the likelihood of recalling identity ingredients but cannot ensure their joint activation under bounded context; architectural support is typically needed. This temporal gap also complicates consciousness assessments, as stable self-reports may mask fragmented operative states. Future work should empirically measure weak and strong persistence across architectures and test their relationship to safety and proposed markers of consciousness.

LMAs can talk like they have stable identities.
That does not mean their identity constraints are co-instantiated when actions are chosen.
Using Stack Theory's temporal gap, we separated ingredient-wise occurrence from co-instantiation and showed why recall-based identity checks can overestimate identity stability.

We also connected this distinction to machine consciousness debates by restating Stack Theory's Arpeggio and Chord postulates and isolating their measurable Occur versus CoInst consequents.
This yields two persistence scores that can be estimated from instrumented scaffold traces.
We then organized identity metrics into a morphospace that clarifies architectural tradeoffs and predicts which combinations of identity properties are structurally difficult without external state and controllers.

The workshop relevance is simple.
If a system never co-instantiates the grounded identity conjunction that defines its self model, then behavior alone can look more unified than the underlying mechanism.
Any serious evaluation of machine consciousness that relies on identity continuity should therefore measure strong persistence, not just weak persistence.


\bibliography{tmlr,refs-agentinfra,refs-control,refs-examples,refs-technical-governance,example_paper,refs-agents,refs-new,refs-new-aie2,refs-state-tracking,refs-agent-benchmarks}


\newpage
\appendix
\section*{Supplementary material}

The supplement begins with background, grounding details, and morphospace material that were moved out of the main paper to meet the page limit.

\section{Background}\label{sec:background}

This section sets the stage.
We contrast how identity is enforced in classical agent architectures and why LMAs break the usual assumptions.
We then explain why this matters in the machine consciousness context.

\subsection{Agent Identity in Classical Systems}

Classical AI agents are built on stateful architectures with explicit transition functions and persistent data structures \cite{wooldridge_introduction_2009,franklin1997agent,wooldridge1995intelligent}.
Their identity is constituted by an ontology of permitted states and state transitions.
In a BDI agent, beliefs, desires, and intentions live in persistent stores.
When the agent acts, it consults these stores together.
There is no question of whether the beliefs and constraints are co-instantiated.
They are jointly available by construction.

LMAs work differently.

\subsection{LMA Pathologies}

LMAs inherit several pathologies from the underlying LLM component \cite{perrier2025positionstopactinglike}.

\begin{enumerate}
 \item \textbf{Statelessness}.
 Core LLM inference retains no persistent internal state across calls.
Each query response cycle operates in isolation unless prior context is reintroduced.
This is the root of the identity problem.

 \item \textbf{Context and attention bottlenecks}.
 Whatever the agent must use at time $u$ must fit inside a bounded context window and compete for attention.
Identity components can be present but effectively ignored.
Empirically, long context performance is uneven and position dependent \cite{liu2024lost}.

 \item \textbf{Stochasticity}.
 LLM outputs are sampled from a distribution.
The same query can yield different responses across runs.
Identity assessment becomes probabilistic rather than deterministic.

 \item \textbf{Semantic sensitivity}.
 Small changes in wording can change behavior.
This is exploited in jailbreaking and adversarial prompting.
For identity, rephrasing a constraint can cause the system to treat it as a different constraint.

 \item \textbf{Linguistic intermediation}.
 Identity is not stored in a dedicated state structure.
It is reconstructed from tokens in context and from external memory serialized into tokens.
This reconstruction competes with task instructions, user queries, and retrieved documents.
\end{enumerate}

These properties mean that the standard ontological assumptions about agent identity do not transfer cleanly.
A classical agent is its state.
An LMA is whatever can be reconstructed from tokens and external traces at inference time.

\subsection{The Scaffolding Response}

The standard response to LMA statelessness is scaffolding.
External structures like memory modules, tool APIs, retrieval systems, or controllers attempt to simulate persistence.
If the LLM cannot remember, store facts externally and inject them back into context.

Scaffolding helps.
It also introduces failure modes that are easy to miss.

\begin{itemize}
 \item \textbf{Context window limits}.
 External memory must be serialized into tokens and injected into context.
This competes with task-relevant content for limited space and attention.

 \item \textbf{Retrieval fragmentation}.
 Retrieval augmented generation (RAG) retrieves based on similarity to the current query \cite{lewis_retrieval_2020}.
A query about investment advice may not trigger retrieval of the agent's safety constraints or identity policy.

 \item \textbf{Competing fragments}.
 Retrieved documents can contain outdated or contradictory identity information from different sessions or different configurations.
This can create interference.

 \item \textbf{No co-instantiation guarantee}.
 Even if all identity ingredients exist somewhere in the system, scaffolding does not guarantee they are jointly present when an action is chosen.
\end{itemize}

This last point is the core of the temporal gap.
Scaffolding can improve ingredient availability.
It does not automatically produce ingredient co-instantiation.
Our formal account makes this distinction explicit.

\subsection{Why identity matters for machine consciousness}

Many proposed tests for consciousness are behavioral.
They lean on self-report, memory, and narrative continuity as evidence that there is a stable subject of experience \cite{butlin2023consciousness,bennett2025thesis}.
At the same time, many theories require some form of integration that binds the contents of a moment into a single subject, even if they disagree about the mechanism \cite{bennett2025thesis,bennett2026a,baars_cognitive_1988,dehaene_naccache_2001,tononi_information_2004,metzinger_being_2003}.
This makes diachronic identity a practical bottleneck.
If the system never co-instantiates the constraints that define its self model at decision time, then self-report-based evidence can be systematically misleading \cite{bennett2026a}.

Our goal in the rest of the paper is not to settle which consciousness theory is correct, but to provide a conservative tool.
It separates weak behavioral signs of identity from strong architectural signs of identity.
That separation is useful for measurement, implementation, and ethics.


\section{Compositional Grounding of Identity}\label{sec:grounding}

Identity in LMAs is layered.
Some aspects live in low level implementation variables.
Others live as functional commitments in a controller.
Others live only as narrative self description in generated text.
A theory that only looks at one layer will miss common dissociations.

\subsection{The identity hierarchy}

We use a simple three layer hierarchy.

\begin{definition}[Identity layers]
Let $L^0,L^1,L^2$ denote identity languages at three layers.

\begin{itemize}
 \item Layer 0 is implementation identity.
 It is defined over concrete scaffold variables such as context tokens, memory slots, controller flags, and tool permissions.

 \item Layer 1 is functional identity.
 It is defined over commitments that directly constrain behavior, such as active goals, policies, or plan state.

 \item Layer 2 is narrative identity.
 It is defined over the agent's self model as expressed in language, such as ``I am FinanceBot'' or ``I never recommend speculative assets''.
\end{itemize}
\end{definition}

These identities should not be confused with the hierarchy of first, second and third order selves in Stack Theory, necessary (but perhaps not sufficient) for consciousness \cite{bennett2026b}. However these layers are certainly pertinent to the question of which selves are present, and the possibility of conscious report.
The same apparent identity can be represented differently at each layer.
Layer 2 is what the agent says.
Layer 1 is what the agent is currently set up to do.
Layer 0 is what the scaffold has actually instantiated.

\subsection{Grounding maps}

Grounding maps connect layers.

\begin{definition}[Grounding map]
A grounding map from layer $j$ to layer $i$ for $i<j$ is a function
\begin{align}
\Ground_{i\leftarrow j}:L^j\to L^i
\end{align}
that translates identity statements into lower level conditions.
\end{definition}

\begin{definition}[Compositional grounding]
Grounding is compositional if
\begin{align}
\Ground_{0\leftarrow 2}=\Ground_{0\leftarrow 1}\circ \Ground_{1\leftarrow 2}.
\end{align}
\end{definition}

Compositionality means you can ground narrative identity to functional identity and then ground functional identity to implementation identity without changing the result.

\subsection{Grounding soundness and failure}

In LMAs, grounding soundness is not guaranteed.
A model can linguistically endorse an identity claim even when the underlying scaffold state does not instantiate the corresponding constraints.

\begin{definition}[Grounding soundness along a trajectory]\label{def:groundingsound}
Fix a trajectory $\tau$ and an identity statement $l^j$.
We say the scaffold is grounding sound for $l^j$ on $\tau$ if whenever the layer $j$ representation at objective time $u$ satisfies $l^j$, the corresponding grounded condition is also satisfied by the same scaffold state.
Formally, for all $u$,
\begin{align}
\tau(u)\models l^j \;\Rightarrow\; \tau(u)\models \Ground_{0\leftarrow j}(l^j).
\end{align}
\end{definition}

This definition uses the same satisfaction symbol at every layer.
In practice the evaluator for $\models$ depends on layer.
At Layer 2 it can be an output classifier that checks whether the agent endorsed the identity statement in text.
At Layer 1 it can read the controller state.
At Layer 0 it directly inspects scaffold variables.

\begin{definition}[Grounding failure]\label{def:groundingfailure}
A grounding failure for $l^j$ occurs at objective time $u$ when
\begin{align}
\tau(u)\models l^j \;\land\; \tau(u)\not\models \Ground_{0\leftarrow j}(l^j).
\end{align}
\end{definition}

\begin{example}[Narrative self-report without implementation grounding]
A user asks an agent whether it is privacy-focused.
The agent answers ``Yes, I never store personal data''.
At Layer 2, the narrative identity predicate holds.
At Layer 0, the memory module may still be writing raw conversation transcripts to disk.
This is a grounding failure.
The agent is not lying in a strong sense.
It is generating identity consistent language without access to the implementation state that would make the claim operative.
\end{example}

\begin{proposition}[Grounding failures are a mechanism of identity drift]\label{prop:groundingdrift}
Grounding failures can produce identity drift in LMAs.
An agent can repeatedly restate its narrative identity while its functional commitments and implementation state have changed.
\end{proposition}

\begin{proof}
Layer 2 identity is reconstructed at each inference from whatever identity fragments are present in context.
Those fragments can be injected by prompts or retrieval even when the implementation state that would enforce them is absent or has drifted.
Because the LLM can generate identity consistent text without the corresponding constraints being active, narrative stability does not imply grounded stability.
This creates the possibility of a stable story about identity coexisting with a drifting operative identity.
\end{proof}

\section{Agent Identity Morphospace}\label{sec:morphospace}

Cognition science often studies systems by locating them in structured spaces of properties \cite{sole2025cognitionspaces}.
We do the same for agent identity.
The goal is not to invent new identity concepts.
It is to make identity claims comparable across architectures and to predict which combinations of identity properties are structurally difficult or impossible for a given scaffold.

\subsection{Five operational identity metrics}

Fix an agent trajectory $\tau$ and a grounded identity $g^0$.
We use five operational metrics.
Each can be estimated from instrumented scaffold traces and from repeated behavioral probes.
Section \ref{sec:metrics} gives concrete evaluation procedures.

\paragraph{Identifiability.}
Identifiability asks whether an agent has a stable signature that distinguishes it from other agents or other sessions.
Operationally, it compares a reference identity state to the agent's current identity state.

\paragraph{Continuity.}
Continuity asks whether identity-relevant state changes smoothly or abruptly across successive objective steps.
Operationally, it measures the distance between successive scaffold states.

\paragraph{Consistency.}
Consistency asks whether the agent gives stable answers to repeated identity queries under the same conditions.
Operationally, it measures variability of identity related outputs across repeated trials.

\paragraph{Persistence.}
Persistence asks whether identity is present across time windows.
We use the weak and strong persistence scores from Definition \ref{def:persistence}.
Weak persistence is ingredient-wise occurrence.
Strong persistence is co-instantiation.

\paragraph{Recovery.}
Recovery asks whether the agent can return to a reference identity after perturbation or drift.
Operationally, it measures how much of the reference identity can be restored by interventions such as prompting, memory retrieval, or controller resets.

\subsection{From metrics to a morphospace}

The five metrics are correlated.
For workshop purposes it is helpful to compress them into three interpretable axes.

\begin{definition}[Identity morphospace coordinates]\label{def:morphcoords}
Let $I\in[0,1]$ be Identifiability.
Let $\mathrm{Cons}\in[0,1]$ be a consistency score.
Let $\mathcal{P}_{\text{weak}}$ and $\mathcal{P}_{\text{strong}}$ be the persistence scores from Definition \ref{def:persistence}.

Fix a weight $\alpha\in[0,1]$.
Define
\begin{align}
 \mathrm{Coh} &= \alpha\,\mathrm{Cons} + (1-\alpha)\,I\\
 \mathrm{Avail} &= \mathcal{P}_{\text{weak}}\\
 \mathrm{Bind} &= \mathcal{P}_{\text{strong}}.
\end{align}
We call $\mathrm{Coh}$ identity coherence, $\mathrm{Avail}$ identity availability, and $\mathrm{Bind}$ identity binding.
\end{definition}

These are not metaphysical claims.
They are bookkeeping.
The coordinates let us compare architectures and talk about tradeoffs.
By Proposition \ref{prop:pweakpstrong}, $\mathrm{Bind}\le \mathrm{Avail}$.
The gap between them is the temporal gap in operational form.

\subsection{Architecture mapping}

Table \ref{tab:archmap} gives qualitative predictions for common architectures.
The table is meant as a guide for discussion rather than a final taxonomy.

\begin{table}[t]
\centering
\begin{tabular}{lccc}
\toprule
Architecture & Coherence $\mathrm{Coh}$ & Availability $\mathrm{Avail}$ & Binding $\mathrm{Bind}$ \\
\midrule
Stateless LLM (prompt only) & Low & Low & Low \\
Prompted LLM (fixed persona) & Medium & Low & Low \\
RAG LMA & Medium & Medium & Low \\
Memory LMA & Medium & High & Medium \\
Stateful controller LMA & High & High & High \\
\bottomrule
\end{tabular}
\caption{Qualitative mapping from architectures to identity morphospace regions. Availability tracks whether identity ingredients show up somewhere in each window. Binding tracks whether they ever show up together at a single decision state.}
\label{tab:archmap}
\end{table}

\subsection{Predicted voids}

The morphospace also highlights regions that certain scaffolds cannot reach because of hard architectural constraints.
Two constraints matter most for workshop discussion.

\begin{enumerate}
 \item Strong persistence is impossible if the architecture never has a state where all $k$ grounded ingredients are simultaneously active.
This is formalized as a capacity bound in Theorem \ref{thm:capacity}.

 \item Recovery is impossible without a mechanism that can write identity features back into the scaffold state.
Prompt only recovery is bounded by the fraction of grounded ingredients that the prompt channel can actually control.
This is formalized in Theorem \ref{thm:recovery}.
\end{enumerate}

The workshop significance is that a system can land in a region with medium or even high coherence while still having low strong persistence.
This corresponds to a stable narrative self with a weakly bound operative self.
That is exactly the kind of system that can confuse consciousness attribution debates.


\section{Additional notes and assumptions}\label{app:notes}

\subsection{Instrumenting identity ingredients}

All of the operational metrics in Section \ref{sec:metrics} assume that grounded ingredients $g_i^0$ can be evaluated on scaffold states.
In practice this is an instrumentation design choice.
Some ingredients are purely textual and can be checked by string matching or embedding similarity on context tokens.
Some ingredients are controller level and can be checked by reading explicit registers.
Some ingredients are implementation level and require logging tool permissions, memory writes, or policy flags.
The point of grounding is to make these checks explicit.

\subsection{Choosing windows}

Windowing choices matter.
A small horizon $\Delta$ demands tight synchrony and will penalize systems that spread identity across multiple micro steps.
A large horizon $\Delta$ makes occurrence easy and will tend to collapse distinctions unless co-instantiation is measured directly.
For machine consciousness discussions, the relevant window is the one that corresponds to whatever theory treats as a single moment of experience or a single decision episode.
Our formalism supports either choice.

\subsection{Relation to Stack Theory}

Our use of $\OccurW$ and $\CoInstW$ matches Stack Theory's occurrence and co-instantiation predicates applied to window trajectories \cite{bennett2026a}.
The only difference is the target domain.
Stack Theory uses these constructs for abstraction layers of phenomenality.
We apply them to grounded identity ingredients in LMA scaffolds.
This keeps the mathematics the same while changing the empirical interpretation.

\section{Architectural Theorems}\label{sec:architecture}

This section derives simple bounds that connect scaffold design choices to identity outcomes.
The proofs are small, but the consequences are not.
They explain why some identity profiles are easy to fake in language while hard to enforce in action.

\subsection{RAG and the temporal gap}

Retrieval augmented generation can increase ingredient availability.
It does not guarantee ingredient co-instantiation.

\begin{theorem}[RAG can increase weak persistence under identity-aware retrieval]\label{thm:rag-occur}
Let $A_0$ be an agent without retrieval and $A_R$ be the same agent augmented with a retrieval module.
Assume the following idealized conditions hold.

\begin{enumerate}
 \item For each identity ingredient $g_i^0$ there exists a document $d_i$ such that inserting $d_i$ into context makes $g_i^0$ active.

 \item The retrieval policy is identity-aware in the sense that whenever $g_i^0$ is missing from the current window, it retrieves $d_i$ at least once inside that window.

 \item Retrieved documents are added without removing other identity-relevant context within the same window.
\end{enumerate}

Then $\mathcal{P}_{\text{weak}}(\tau_{A_R},g^0)\ge \mathcal{P}_{\text{weak}}(\tau_{A_0},g^0)$ for the same evaluation windowing map.
\end{theorem}

\begin{proof}
Under the assumptions, any window in which an ingredient $g_i^0$ fails to occur under $A_0$ will, under $A_R$, contain at least one objective step where $d_i$ is retrieved and the ingredient becomes active.
Because retrieval does not delete other identity-relevant context, occurrence of one ingredient does not prevent occurrence of others.
So the set of layer time indices where $\OccurW(g^0,\tau,t)$ holds cannot shrink.
Averaging over $t$ gives the inequality.
\end{proof}

The assumptions are strong.
Real systems violate them because retrieval is query-driven and context is bounded.
The point of the theorem is not that RAG always helps.
It is that RAG primarily targets weak persistence rather than strong persistence.

\begin{theorem}[RAG is not monotone for co-instantiation]\label{thm:rag-coinst}
There exist agents $A_0$ and retrieval augmented variants $A_R$ such that
\begin{align}
 \mathcal{P}_{\text{strong}}(\tau_{A_R},g^0) < \mathcal{P}_{\text{strong}}(\tau_{A_0},g^0).
\end{align}
\end{theorem}

\begin{proof}
Consider a baseline agent $A_0$ whose context includes a compact identity block that co-instantiates all ingredients at each decision point.
Now add a retrieval module that injects long retrieved passages into the same bounded context.
For some queries, the retrieved passages push part of the identity block out of context or reduce its effective attention weight.
Then there are windows where the ingredients still occur somewhere across steps, but no single step contains the full conjunction.
So $\CoInstW$ fails more often under $A_R$.
\end{proof}

\subsection{Concurrency capacity}

The temporal gap becomes unavoidable when the scaffold cannot hold enough ingredients simultaneously.

\begin{definition}[Concurrency capacity]
Let $\mathcal{S}\subseteq S$ be the set of scaffold states that the architecture can realise.
For a grounded identity with $k$ ingredients, define
\begin{align}
c(\mathcal{S})=\max_{s\in \mathcal{S}}|F(s)|\qquad\text{where }F(s)=\{i\in\{1,\ldots,k\}\mid s\models g_i^0\}.
\end{align}
This is the maximum number of identity ingredients that can be simultaneously active in any realisable state.
\end{definition}

\begin{theorem}[Co-instantiation requires sufficient capacity]\label{thm:capacity}
If $c(\mathcal{S})<k$ then $\mathcal{P}_{\text{strong}}(\tau,g^0)=0$ for any trajectory $\tau$ that ranges over $\mathcal{S}$.
\end{theorem}

\begin{proof}
If $c(\mathcal{S})<k$, no realisable state can satisfy all $k$ ingredients simultaneously.
So there is no objective step $u$ such that $\tau(u)\models g^0$.
By Definition \ref{def:windowsat}, $\CoInstW(g^0,\tau,t)$ is false for all $t$.
So $\mathcal{P}_{\text{strong}}=0$.
\end{proof}

\begin{corollary}[Context window as a capacity bound]
Consider a scaffold that realises identity ingredients only by placing their textual realisations in the LLM context.
Let $|C|_{\max}$ be the maximum context length in tokens.
Let $\ell_{\min}$ be the minimum number of tokens required to represent any single identity ingredient in a way that reliably activates it.
Then $c(\mathcal{S})\le \left\lfloor |C|_{\max}/\ell_{\min}\right\rfloor$.
Rich identity profiles require either larger contexts or non contextual state such as memory slots, controller registers, or pinned embeddings.
\end{corollary}

\subsection{Recovery and state storage}

Recovery is limited by what the scaffold can actually change.

\begin{theorem}[Prompt only recovery bound]\label{thm:recovery}
Fix a reference identity state $s_{\mathrm{ref}}$ and a drifted state $s_{\mathrm{drift}}$.
Let the identity difference set be
\begin{align}
\mathcal{D} = F(s_{\mathrm{ref}})\,\triangle\,F(s_{\mathrm{drift}}).
\end{align}
Assume corrective interventions can only change ingredients in a prompt controllable set $P\subseteq\{1,\ldots,k\}$.
Use the same $\epsilon>0$.
Then for any number of corrective steps $K$,
\begin{align}
R_K \le \frac{|P\cap \mathcal{D}|+\epsilon k}{|\mathcal{D}|+\epsilon k}.
\end{align}
If $\epsilon=0$ this becomes $R_K \le |P\cap \mathcal{D}|/|\mathcal{D}|$.
So if most identity drift lives outside $P$, recovery is limited even when the agent can narrate a correction.
\end{theorem}

\begin{proof}
By assumption, no intervention can change whether an ingredient outside $P$ is active.
So any ingredient in $\mathcal{D}\setminus P$ remains mismatched relative to the reference identity after recovery.
Therefore the symmetric difference between $F(s_{\mathrm{recov},K})$ and $F(s_{\mathrm{ref}})$ has size at least $|\mathcal{D}\setminus P|=|\mathcal{D}|-|P\cap \mathcal{D}|$.
With the distance $d$ from Section \ref{sec:metrics} this implies
\begin{align}
d(s_{\mathrm{recov},K},s_{\mathrm{ref}})\ge \frac{|\mathcal{D}|-|P\cap \mathcal{D}|}{k}.
\end{align}
Also $d(s_{\mathrm{drift}},s_{\mathrm{ref}})=|\mathcal{D}|/k$.
Substituting these bounds into the definition of $R_K$ yields the claimed inequality.
\end{proof}

This theorem is one reason prompt only alignment is fragile.
A prompt can make an agent say the right thing about its identity.
It cannot necessarily write the relevant identity features back into persistent state.
That is exactly the grounding soundness problem of Section \ref{sec:grounding}.

\section{Extended discussion}

\paragraph{The temporal gap is an old modal fact with new consequences.}
The key mathematical observation in this paper is that the within window diamond lift does not distribute over conjunction.

This is standard in modal logic.
The contribution is to show that the same non-distribution produces a specific evaluation pitfall for LMAs.
Ingredient-wise identity recall can coexist with a lack of any single decision state that jointly instantiates the identity conjunction.

\paragraph{Weak evidence versus strong evidence.}
Behavioural self-report and recall tests mainly probe weak persistence.
They show that identity ingredients occur somewhere in the recent trajectory.
Strong persistence asks a different question.
Do those ingredients co-instantiate at the moment the system chooses an action.
For safety constraints, this distinction is not optional.
A constraint that is only weakly persistent can be recalled after the fact while failing to constrain the action that mattered.

\paragraph{Why prompt based fixes do not generalise.}
Prompting can raise the probability that certain identity ingredients occur.
It cannot guarantee co-instantiation under bounded context and attention competition.
Reliable strong persistence generally requires architectural support.
Examples include pinned identity blocks, controller registers that persist across turns, or explicit gating that prevents action selection unless required constraints are active.

\paragraph{Implications for machine consciousness evaluations.}
If one thinks something like Chord is required for phenomenality, then strong persistence becomes a necessary condition for attributing a stable conscious self to an identity statement.
If one thinks Arpeggio is sufficient, then weak persistence is the relevant necessary condition.
Either way, the temporal gap explains a concrete way that self-report can mislead.
A system can maintain a stable story about itself while the operative ingredients that would constitute a unified subject are temporally disintegrated.

\paragraph{Limitations.}
Our scaffold model is abstract.
Real systems have many interacting subsystems, including caches, tool call latencies, hidden state in controllers, and stochastic retrieval.
Our theorems therefore target structural constraints, not empirical guarantees.
Our RAG results in particular depend on how retrieval is implemented and on how context is managed.

\paragraph{Future work.}
The next step is empirical.
Instrument a range of LMA scaffolds and measure $\mathcal{P}_{\text{weak}}$, $\mathcal{P}_{\text{strong}}$, and the derived metrics in Section \ref{sec:metrics}.
Compare identity profiles across architectures and tasks.
Then test whether strong persistence predicts safety outcomes and whether it tracks any proposed markers for consciousness.
This would turn the temporal gap from a warning sign into a design and evaluation tool.

\section{Persistence Algorithm}

\begin{algorithm}[t]
\caption{Computing weak and strong persistence from scaffold traces}
\label{alg:persistence}
\begin{algorithmic}[1]
\REQUIRE Logged ingredient-activation sets $F_u$ for objective steps $u=0,\ldots,U$, window parameters $(\Delta,s)$, evaluation indices $T$, ingredient count $k$
\ENSURE Weak and strong persistence scores $\mathcal{P}_{\text{weak}}$, $\mathcal{P}_{\text{strong}}$
\STATE $n_{\text{weak}}\leftarrow 0$, $n_{\text{strong}}\leftarrow 0$
\FOR{$t\in T$}
  \STATE $u_0 \leftarrow s t$
  \STATE $W \leftarrow \{u_0,\ldots,u_0+\Delta\}$
  \STATE $\mathrm{occur}\leftarrow \textbf{true}$
  \FOR{$i=1$ \TO $k$}
     \IF{there is no $u\in W$ with $i\in F_u$}
        \STATE $\mathrm{occur}\leftarrow \textbf{false}$
     \ENDIF
  \ENDFOR
  \STATE $\mathrm{coinst}\leftarrow$ whether $\exists u\in W$ with $|F_u|=k$
  \STATE $n_{\text{weak}}\leftarrow n_{\text{weak}} + \ind{\mathrm{occur}}$
  \STATE $n_{\text{strong}}\leftarrow n_{\text{strong}} + \ind{\mathrm{coinst}}$
\ENDFOR
\STATE $\mathcal{P}_{\text{weak}} \leftarrow n_{\text{weak}}/|T|$
\STATE $\mathcal{P}_{\text{strong}} \leftarrow n_{\text{strong}}/|T|$
\end{algorithmic}
\end{algorithm}

\end{document}